\documentclass[11pt]{article}

\usepackage{fullpage,times,url,natbib}

\usepackage{amsthm,amsfonts,amsmath,amssymb,epsfig,color,float,graphicx,verbatim}
\usepackage{algorithm,algorithmic}

\usepackage{hyperref}
\hypersetup{
	colorlinks   = true, 
	urlcolor     = blue, 
	linkcolor    = blue, 
	citecolor   = black 
}

\newtheorem{theorem}{Theorem}

\newtheorem{lemma}{Lemma}
\newtheorem{corollary}{Corollary}
\newtheorem{definition}{Definition}
\newtheorem{remark}{Remark}

\newcommand{\E}{\mathbb{E}}

\newcommand{\ba}{\mathbf{a}}
\newcommand{\be}{\mathbf{e}}
\newcommand{\bx}{\mathbf{x}}
\newcommand{\bw}{\mathbf{w}}

\newcommand{\bu}{\mathbf{u}}
\newcommand{\bv}{\mathbf{v}}
\newcommand{\bz}{\mathbf{z}}

\newcommand{\bd}{\mathbf{d}}

\newcommand{\by}{\mathbf{y}}
\newcommand{\br}{\mathbf{r}}

\newcommand{\Ocal}{\mathcal{O}}

\newcommand{\Xcal}{\mathcal{X}}

\newcommand{\Wcal}{\mathcal{W}}

\newcommand{\norm}[1]{\|#1\|}

\newtheorem{example}{Example}
\newtheorem{assumption}{Assumption}

\newcommand{\secref}[1]{Sec.~\ref{#1}}

\renewcommand{\eqref}[1]{Eq.~(\ref{#1})}
\newcommand{\lemref}[1]{Lemma~\ref{#1}}

\newcommand{\thmref}[1]{Thm.~\ref{#1}}

\title{Are ResNets Provably Better than Linear Predictors?}
\author{Ohad Shamir\\Weizmann Institute of 
	Science and Microsoft Research}
\date{}

\begin{document}
	
\maketitle

\begin{abstract}
A residual network (or ResNet) is a standard deep neural net architecture, 
with state-of-the-art performance across numerous applications. The main 
premise of ResNets is that they allow the training of each layer to focus on 
fitting just the residual of the previous layer's output and the target output. 
Thus, we should expect that the trained network is no worse than what we can 
obtain if we remove the residual layers and train a shallower network instead. 
However, due to the non-convexity of the optimization problem, it is not at all 
clear that ResNets indeed achieve this behavior, rather than getting stuck at 
some arbitrarily poor local minimum. In this paper, we rigorously prove that 
arbitrarily deep, nonlinear residual units indeed exhibit this behavior, in the 
sense that the optimization landscape contains no local minima with value above 
what can be obtained with a linear predictor (namely a 1-layer network). 
Notably, we show this under minimal or no assumptions on the precise 
network architecture, data distribution, or loss function used. We also provide 
a quantitative analysis of approximate stationary points for this 
problem. Finally, we show that with a certain tweak to the architecture, 
training the network with standard stochastic gradient descent achieves an 
objective value close or better than any linear predictor.
\end{abstract}

\section{Introduction}

Residual networks (or ResNets) are a popular class 
of artificial neural networks, providing state-of-the-art performance across 
numerous applications
\citep{he2016deep,he2016identity,kim2016accurate,xie2017aggregated,xiong2017microsoft}.
Unlike vanilla feedforward neural networks, ResNets are characterized by skip 
connections, in which the output of one layer is directly added to the output 
of some following layer. Mathematically, whereas feedforward neural networks 
can be expressed as stacking layers of the form
\[
\by = g_{\Phi}(\bx)~,
\]
(where $(\bx,\by)$ is the input-output pair and $\Phi$ are the tunable  
parameters of the function $g_{\Phi}$), ResNets are built from 
``residual units'' of 
the form $\by = f\left(h(\bx)+g_{\Phi}(\bx)\right)$, where $f,h$ are fixed 
functions. In fact, it is common to let $f,h$ be the 
identity \citep{he2016identity}, in which case each unit takes the form
\begin{equation}\label{eq:resnet}
\by = \bx+g_{\Phi}(\bx)~.
\end{equation}
Intuitively, this means that in each layer, the training of $f_{\Phi}$ can 
focus on fitting just the ``residual'' of the target $\by$ given 
$\bx$, rather than $\by$ itself. In particular, adding more depth should not 
harm performance, since we can effectively eliminate layers by tuning 
$\Phi$ such that $g_{\Phi}$ is the zero function. Due to this property, 
residual networks have proven to 
be very effective in training extremely deep networks, with hundreds of layers 
or more. 

Despite their widespread empirical success, our rigorous theoretical 
understanding of training residual networks is very limited. Most recent 
theoretical works on optimization in deep learning (e.g. 
\citet{soltanolkotabi2017theoretical,yun2018critical,soudry2017exponentially,
	brutzkus2017sgd,ge2017learning,safran2017spurious,du2018power} to name just 
	a few examples) have focused 
on simpler, feedforward architectures, which 
do not capture the properties of residual networks. Some recent 
results do consider residual-like elements (see discussion of related work 
below), but generally do not apply to standard architectures. In particular, we 
are not aware of any theoretical justification for the basic premise of 
ResNets: Namely, that their architecture allows adding layers without harming 
performance. The problem is that training neural networks involves solving a 
highly non-convex problem using local search procedures. Thus, even though 
deeper residual networks can \emph{express} shallower ones, it is not at 
all clear that the training process will indeed converge to such a network (or 
a better one). Perhaps, when we attempt to train the residual network using 
gradient-based methods, we might hit some poor local minimum, with a worse 
error than what can be obtained with a shallower network? This question is the 
main motivation to our work.

A secondary motivation are several recent results (e.g. 
\cite{yun2018critical,safran2017spurious,du2017gradient,liang2018understanding}),which
 demonstrate how spurious local minima (with value larger than the global
minima) do exist in general when training neural networks, even under fairly 
strong assumptions. Thus, instead of aiming for a result demonstrating that no 
such minima exist, which might be too good to be true on realistic networks, we 
can perhaps consider a more modest goal, showing that no such minima exist 
above a certain (non-trivial) level set. This level set can correspond, for 
instance, to the optimal value attainable by shallower networks, without the 
additional residual layers.

In this paper, we study these questions by considering the competitiveness of a
simple residual network (composed of an arbitrarily deep, nonlinear residual 
unit and a linear output layer) with respect to linear predictors (or 
equivalently, 1-layer networks). Specifically, we consider the optimization 
problem associated with training such a residual network, which is in general 
non-convex and can have a complicated structure. Nevertheless, we prove that 
the optimization landscape has no local minima \emph{with 
a value higher} than what can be achieved with a linear predictor on the same 
data. In other words, if we run a local search procedure and reach a local 
minimum, we are assured that the solution is no worse than the best 
obtainable with a linear predictor. Importantly, we show this under fairly 
minimal assumptions on the residual unit, no assumptions on the data 
distribution (such as linear separability), and no assumption on the loss 
function used besides smoothness and convexity in the network's output (which 
is satisfied for losses used in practice). In addition, we provide a 
quantitative analysis, which shows how every point which is $\epsilon$-close to 
being stationary in certain directions (see \secref{sec:setting} for a precise 
definition) can't be more than $\text{poly}(\epsilon)$ worse than any fixed 
linear predictor.

The results above are geometric in nature. As we explain later on, they do not 
necessarily imply that standard gradient-based methods will indeed converge to 
such desirable solutions (for example, since the iterates might diverge). 
Nevertheless, we also provide an algorithmic result, 
showing that if the residual architecture is
changed a bit, then a standard stochastic gradient descent (SGD) procedure will 
result in a predictor similar or better than the best linear predictor. This 
result relies on a simple, but perhaps unexpected reduction to the setting of 
online learning, and might be of independent interest. 

Our paper is structured as follows. After discussing related work below, we 
formally define our setting and notation in \secref{sec:setting}. In 
\secref{sec:competitiveness}, we present our main results, showing how our 
residual networks have no spurious local minima above the level set obtainable 
with linear predictors. In \secref{sec:norm}, we discuss the challenges in 
translating this geometric result to an algorithmic one (in particular, the 
possibility of suboptimal approximate stationary points far enough from the 
origin). In \secref{sec:skipoutput} we do provide a positive algorithmic 
result, assuming that the network architecture is changed a bit. All proofs are 
provided in \secref{sec:proofs}. Finally, in Appendix \ref{app:vec}, we 
discuss a generalization of our results to vector-valued outputs.

\subsection*{Related Work}

As far as we know, existing rigorous theoretical results on residual networks 
all pertain to linear networks, which combine linear residual units of the form
\[
\by ~=~ \bx+W\bx ~=~ (I+W)\bx~.
\]
Although such networks are not used in practice, they capture important aspects 
of the non-convexity associated with training residual networks. In particular, 
\citet{hardt2016identity} showed that linear residual networks with the squared 
loss have no spurious local minima (namely, every local minimum is also a 
global one). More recently, \citet{bartlett2018gradient} proved convergence 
results for gradient descent on such problems, assuming the inputs are 
isotropic and the target linear mapping is symmetric and positive definite. 
Showing similar results for non-linear networks is mentioned in  
\citet{hardt2016identity} as a major open problem. In our paper, we focus on 
non-linear residual units, but consider only local minima above some level set. 

In terms of the setting, perhaps the work closest to ours is 
\citet{liang2018understanding}, 
which considers networks which can be written as  
$\bx\mapsto f_S(\bx)+f_D(\bx)$, where $f_S$ is a one-hidden-layer network, and 
$f_D$ is an arbitrary, possibly deeper 
network. Under technical assumptions on the data distribution, activations 
used, network size, and assuming certain classification losses, the authors 
prove that the training objective is benign, in the 
sense that the network corresponding to any local minimum has zero 
classification error. However, as the authors point out, their architecture is 
different than standard ResNets (which would require a final tunable layer to 
combine the outputs of $f_S,f_D$), and their results provably do not hold under 
such an architecture. Moreover, the technical 
assumptions are non-trivial, do not apply as-is to standard activations and 
losses (such as the ReLU activation and the logistic loss), and require 
specific conditions on the data, such as linear separability or a certain 
low-rank structure. In contrast, we study a more standard residual unit, and 
make minimal or no assumptions on the network, data distribution, and loss 
used. On the flip side, we only prove results for local minima above a certain 
level set, rather than all such points. 

Finally, the idea of studying stationary points in non-convex optimization 
problems, which are above or below some reference level set, has also been 
explored in some other works (e.g. \cite{ge2017optimization}), but under 
settings quite different than ours.

\section{Setting and Preliminaries}\label{sec:setting}

We start with a few words about basic notation and terminology. We 
generally use bold-faced letters to denote vectors (assumed to be in column 
form), and capital letters to 
denote matrices or functions. $\norm{\cdot}$ refers to the Euclidean norm for 
vectors and spectral norm for matrices, unless specified otherwise. 
$\norm{\cdot}_{Fr}$ for matrices denotes the Frobenius norm (which always upper 
bounds the spectral norm). For a matrix $M$, $\text{vec}(M)$ refers to the 
entries of $M$ written as one long vector (according to some canonical order).  
Given a function $g$ on Euclidean space, $\nabla g$ 
denotes its gradient and $\nabla^2 g$ denotes its 
Hessian. A point $\bx$ in the domain of a 
function $g$ is a local minimum, if $g(\bx)\leq g(\bx')$ for any $\bx'$ in some 
open neighborhood of $\bx$. Finally, we use standard $\Ocal(\cdot)$ and 
$\Theta(\cdot)$ notation to 
hide constants, and let $\text{poly}(\bx_1,\ldots,\bx_r)$ refer
to an expression which is polynomial in $\bx_1,\ldots,\bx_r$. 

We consider a residual network architecture, consisting of a residual unit as 
in \eqref{eq:resnet}, composed with a linear output layer, with scalar 
output\footnote{See Appendix \ref{app:vec} for a discussion of how some of our 
results can be generalized to networks with vector-valued outputs.}:
\[
\bx~\mapsto~\bw^\top\left(\bx+g_{\Phi}(\bx)\right).
\]
We will make no assumptions on the structure of each $g_{\Phi}$, nor on the 
overall depth of the network which computes it, except that it's last layer is 
a tunable linear transformation (namely, that 
$g_{\Phi}(\bx)=Vf_{\theta}(\bx)$ for some matrix $V$, not necessarily a 
square one, and parameters $\theta$). This condition follows the ``full 
pre-activation'' structure proposed in \citet{he2016identity}, which was 
empirically found to be the 
best-performing residual unit architecture, and is commonly used in practice 
(e.g. in TensorFlow). We depart from that structure only in that $V$ is fully 
tunable 
rather than a convolution, to facilitate and simplify our theoretical study. 
Under this assumption, we have that given $\bx$, the network outputs
\[
\bx~\mapsto \bw^\top\left(\bx+Vf_{\theta}(\bx)\right)~,
\]
parameterized by a vector $\bw$, a matrix $V$, and with some (possibly 
complicated) function $f_{\theta}$ parameterized by $\theta$.

\begin{remark}[Biases]
We note that 
this model can easily incorporate biases, namely predictors of the form 
$\bx~\mapsto 
\bw^\top\left(\bx+Vf_{\theta}(\bx)+\ba\right)+a$ for some tunable $a,\ba$, by 
the standard trick of augmenting $\bx$ with an additional coordinate whose 
value is always $1$, and assuming that $f_{\theta}(\bx)$ outputs a vector with 
an additional coordinate of value $1$. Since our results do not depend on the 
data geometry or specifics of $f_{\theta}$, they would not be affected by such 
modifications. 
\end{remark}

We assume that our network is trained with respect to some data 
distribution (e.g. an average over some training set $\{\bx_i,y_i\}$), using a 
loss function $\ell(p,y)$, where $p$ is the network's prediction and $y$ is the 
target value. Thus, we consider the optimization problem
\begin{equation}\label{eq:obj}
\min_{\bw,V,\theta}~F(\bw,V,\theta) ~:=~
\E_{\bx,y}\left[\ell(\bw^\top(\bx+V f_{\theta}(\bx));y)\right]~,
\end{equation}
where $\bw,V,\theta$ are unconstrained. This objective will be the main focus 
of our paper.  In general, this objective is not convex in
$(\bw,V,\theta)$, and can easily have spurious local minima and saddle points.

In our results, we will make no explicit assumptions on the distribution of 
$(\bx,y)$, nor on the structure of $f_{\theta}$. As to the loss, we will assume 
throughout the paper the following:
\begin{assumption}\label{ass:diff}
	For any $y$, the loss $\ell(p,y)$ is twice differentiable and convex in $p$.
\end{assumption}
This assumption is mild, and is satisfied for standard losses such as the 
logistic loss, squared loss, smoothed hinge loss etc. Note that under this 
assumption, $F(\bw,V,\theta)$ is twice-differentiable with respect to $\bw,V$, 
and in particular the function defined as
\[
F_{\theta}(\bw,V)~:=~F(\bw,V,\theta)
\]
(for any fixed $\theta$) is twice-differentiable. We emphasize that throughout 
the paper, we will \emph{not} assume that $F$ is necessarily differentiable 
with respect to $\theta$ (indeed, if $f_{\theta}$ represents a network with  
non-differentiable 
operators such as ReLU or the $\max$ function, we cannot expect 
that $F$ will be differentiable 
everywhere). When 
considering derivatives of $F_{\theta}$, we think of the input as one long 
vector in Euclidean space (in order specified by $\text{vec}()$), so $\nabla 
F_{\theta}$ is a vector and $\nabla^2 
F_{\theta}$ is a matrix. 

As discussed in the introduction, we wish to compare our objective value to 
that obtained by linear predictors. Specifically, we will use the notation
\[
F_{lin}(\bw) ~:=~ F(\bw,\mathbf{0},\theta) = 
\E_{\bx,y}\left[\ell(\bw^\top\bx;y)\right]
\]
to denote the expected loss of a \emph{linear} predictor parameterized by the 
vector $\bw$. By Assumption \ref{ass:diff}, this function is convex and
twice-differentiable.

Finally, we introduce the following class of points, which behave approximately 
like local minima of $F$ with respect to $(\bw,V)$, in terms of its first two 
derivatives:

\begin{definition}[$\epsilon$-SOPSP] 
Let $\mathcal{M}$ be an open subset of the domain 
of  $F(\bw,V,\theta)$, on which $\nabla^2 F_{\theta}(\bw,V)$ is 
$\mu_2$-Lipschitz in $(\bw,V)$. Then 
$(\bw,V,\theta)\in\mathcal{M}$ is an 
\emph{$\epsilon$-second-order partial stationary point ($\epsilon$-SOPSP)} of 
$F$ on $\mathcal{M}$, if 
\[
\norm{\nabla F_{\theta}(\bw,V)}\leq \epsilon~~~\text{and}~~~
\lambda_{\min}(\nabla^2 F_{\theta}(\bw,V))\geq 
-\sqrt{\mu_2 \epsilon}~.
\]  
\end{definition}%
Importantly, note that \emph{any} local minimum $(\bw,V,\theta)$ of $F$ 
must be 
a $0$-SOPSP: This is because $(\bw,V)$ is a local minimum of the 
(differentiable) function $F_{\theta}$, 
hence $\norm{\nabla F_{\theta}(\bw,V)}=0$ and $\lambda_{\min}(\nabla^2 
F_{\theta}(\bw,V))\geq 0$. Our definition above directly generalizes the 
well-known notion of $\epsilon$-second-order stationary points (or 
$\epsilon$-SOSP) 
\citep{mccormick1977modification,nesterov2006cubic,jin2017escape}, which are 
defined for functions which are twice-differentiable in all of their 
parameters. 
In fact, our definition of $\epsilon$-SOPSP is equivalent to requiring 
that $(\bw,V)$ is an $\epsilon$-SOSP of $F_{\theta}$. We need to use this more 
general definition, because we are not assuming that $F$ is differentiable in 
$\theta$. Interestingly, $\epsilon$-SOSP is one of the most 
general classes of points in non-convex optimization, to which gradient-based 
methods can be shown to converge in $\text{poly}(1/\epsilon)$ iterations. 

\section{Competitiveness with Linear Predictors}\label{sec:competitiveness}

Our main results are \thmref{thm:mainstat} and Corollary \ref{cor:nolocal} 
below, 
which are proven in two stages: First, 
we show that at any point such that $\bw\neq \mathbf{0}$, $\norm{\nabla 
F_{\theta}(\bw,V)}$ is 
lower bounded in terms of the suboptimality with respect to the best linear 
predictor (\thmref{thm:pl}). We then consider the case $\bw=\mathbf{0}$, and 
show that for such points, if they are suboptimal with respect to the best 
linear predictor, then \emph{either} $\norm{\nabla F_{\theta}(\bw,V)}$ is 
strictly positive, \emph{or} $\lambda_{\min}(\nabla^2 F_{\theta}(\bw,V))$ 
is strictly negative (\thmref{thm:atzero}). Thus, building on the definition of 
$\epsilon$-SOPSP from the previous section, we can show that no point which is 
suboptimal (compared to a linear predictor) can be a local minimum of $F$. 

\begin{theorem}\label{thm:pl}
	At any point $(\bw,V,\theta)$ such that $\bw\neq\mathbf{0}$, and for any	
	vector $\bw^*$ of the same dimension as $\bw$, 
	\[
	\norm{\nabla F_{\theta}(\bw,V)} ~\geq~ 
	\frac{F(\bw,V,\theta)-F_{lin}(\bw^*)}{\sqrt{2\norm{\bw}^2+\norm{\bw^*}^2\left(2+\frac{\norm{V}^2}
			{\norm{\bw}^2}\right)}}~.
	\]
\end{theorem}

The theorem implies that for any point $(\bw,V,\theta)$ for which the 
objective 
value $F(\bw,V,\theta)$ is larger than that of some linear predictor 
$F_{lin}(\bw^*)$, and unless $\bw=\mathbf{0}$, its partial derivative with 
respect to $(\bw,V)$ (namely $\nabla F_{\theta}(\bw,V)$) is non-zero, so 
it cannot be a stationary point with respect to $\bw,V$, nor a local minimum of 
$F$. The proof (in 
\secref{sec:proofs}) relies on showing that 
the inner product of $\nabla F_{\theta}(\bw,V)$ with a certain 
carefully-chosen vector can be lower bounded by 
$F(\bw,V,\theta)-F_{lin}(\bw^*)$. 

To analyze the case $\bw=\mathbf{0}$, we have the following result:

\begin{theorem}\label{thm:atzero}
	For any 
	$V,\theta,\bw^*$, 
	\[
	\lambda_{\min}\left(\nabla^2 F_{\theta}(\mathbf{0},V)\right)~\leq~0
	\]
	and
	\begin{align*}
	\norm{\nabla
		F_{\theta}(\mathbf{0},V)}+\norm{V}&
		\sqrt{\left|\lambda_{\min}\left(\nabla^2 
			F_{\theta}(\mathbf{0},V)\right)\right|\cdot 
				\left\|\frac{\partial^2}{\partial\bw^2}F_{\theta}(\mathbf{0},V)
				\right\|+\lambda_{\min}
			\left(\nabla^2 F_{\theta}(\mathbf{0},V)\right)^2}\\
	&~\geq~ \frac{F(\mathbf{0},V,\theta)-F_{lin}(\bw^*)}{\norm{\bw^*}}~,
	\end{align*}
	where 
	$\lambda_{\min}(M)$ denotes the minimal eigenvalue of a symmetric matrix 
	$M$. 
\end{theorem}

Combining the two theorems above, we can show the following main result:
\begin{theorem}\label{thm:mainstat}
	Fix some positive 
	$b,r,\mu_0,\mu_1,\mu_2$ and $\epsilon\geq 0$, and suppose $\mathcal{M}$ is 
	some convex open subset 
	of the domain of $F(\bw,V,\theta)$ in which 
	\begin{itemize}
		\item $\max\{\norm{\bw},\norm{V}\}\leq b$
		\item $F_{\theta}(\bw,V)$, $\nabla F_{\theta}(\bw,V)$ and $\nabla^2 
		F_{\theta}(\bw,V)$  are $\mu_0$-Lipschitz, $\mu_1$-Lipschitz, and 
		$\mu_2$-Lipschitz in $(\bw,V)$ respectively.
		\item For any $(\bw,V,\theta)\in\Wcal$, we have 
				$(\mathbf{0},V,\theta)\in 
				\Wcal$ and  $\norm{\nabla^2 F_{\theta}(\mathbf{0},V)}\leq 
				\mu_1$. 
	\end{itemize}
	Then for any $(\bw,V,\theta)\in	\mathcal{M}$ which is an 
	$\epsilon$-SOPSP of $F$ on $\mathcal{M}$,
	\[
	F(\bw,V,\theta)~\leq~ \min_{\bw:\norm{\bw}\leq r} 
	F_{lin}(\bw)+(\epsilon+\sqrt[4]{\epsilon})\cdot\text{poly}(b,r,\mu_0,\mu_1,\mu_2).
	\]
\end{theorem}
We note that the $\text{poly}(b,r,\mu_0,\mu_1,\mu_2)$ term hides only 
dependencies which are at most linear in the individual factors (see the proof 
in \secref{sec:proofs} for the exact expression).

As discussed in \secref{sec:setting}, any local minima of $F$ must correspond 
to a $0$-SOPSP. Hence, the theorem above implies that 
for such a point, $F(\bw,V,\theta)\leq \min_{\bw:\norm{\bw}\leq r} 
F_{lin}(\bw)$ (as long as $F$ satisfies the Lipschitz continuity assumptions 
for some finite $\mu_0,\mu_1,\mu_2$ on 
any bounded subset of the domain). Since this 
holds for any $r$, we have arrived at the following corollary:

\begin{corollary}\label{cor:nolocal}
	Suppose that on any bounded subset of the domain of $F$, it holds that 
	$F_{\theta}(\bw,V),\nabla F_{\theta}(\bw,V)$ and $\nabla^2 
	F_{\theta}(\bw,V)$ are all Lipschitz continuous in $(\bw,V)$. Then every 
	local minimum $(\bw,V,\theta)$ of 
	$F$ satisfies
	\[
	F(\bw,V,\theta)~\leq~ \inf_{\bw}F_{lin}(\bw)~.
	\]
\end{corollary}
In other words, the objective $F$ has no spurious local minima with value above 
the smallest attainable with a linear predictor.

\begin{remark}[Generalization to vector-valued outputs]
	One can consider a generalization of our setting to networks with 
	vector-valued outputs, namely $\bx\mapsto W(\bx+Vf_{\theta}(\bx))$, where 
	$W$ is a matrix, and with losses $\ell(\mathbf{p},\by)$ taking 
	vector-valued arguments and convex in $\mathbf{p}$ (e.g. the cross-entropy 
	loss). In this more general setting, it is possible 
	to prove a variant of \thmref{thm:pl} using a similar proof technique (see 
	Appendix \ref{app:vec}). However, 
	it is not clear to us how to prove an analog of \thmref{thm:atzero} and 
	hence \thmref{thm:mainstat}. We leave this as a question for future 
	research.
\end{remark}

\section{Effects of Norm and Regularization}\label{sec:norm}

\thmref{thm:mainstat} implies that any $\epsilon$-SOPSP must have a value not 
much worse than that obtained by a linear predictor. Moreover, as discussed in 
\secref{sec:setting}, such points are closely related to second-order 
stationary points, and gradient-based methods 
are known to converge quickly to such points (e.g. 
\citet{jin2017escape}). Thus, it is tempting to claim that such methods will 
indeed result in a network competitive with linear predictors. Unfortunately, 
there is a fundamental catch: The bound of \thmref{thm:mainstat} depends on the 
norm of the point (via $\norm{\bw},\norm{V}$), and can be 
arbitrarily bad if the norm is sufficiently large. In other words, 
\thmref{thm:mainstat} guarantees that a point which is $\epsilon$-SOPSP 
is only ``good'' as long as it is not too far away from the origin. 

If the dynamics of the gradient method are such that the iterates remain in 
some bounded domain (or at least have a sufficiently slowly increasing norm), 
then this would not be an issue. However, we are not a-priori guaranteed that 
this would be the case: 
Since the optimization problem is unconstrained, and we are not assuming 
anything on the structure of $f_{\theta}$, it could be that the 
parameters $\bw,V$ diverge, and no meaningful 
algorithmic result can be derived from \thmref{thm:mainstat}. 

Of course, one option is that this dependence on $\norm{\bw},\norm{V}$ is an 
artifact of the analysis, and any $\epsilon$-SOPSP of 
$F$ is competitive with a linear predictor, regardless of the norms. However, 
the following example shows that this is not the case:

\begin{example}\label{example:badstat}
	Fix some $\epsilon>0$. Suppose $\bx,\bw,V,\bw^*$ are all scalars, 
	$\bw^*=1$, 
	$f_{\theta}(\bx)=\epsilon \bx$ (with no dependence on a parameter 
	$\theta$), $\ell(p;y)=\frac{1}{2}(p-y)^2$ is the squared loss, and 
	$\bx=y=1$ w.p. 
	$1$. Then the objective can be equivalently written as
	\[
	F(w,v) = \frac{1}{2}\left(w(1+\epsilon v)-1\right)^2
	\]
	(see leftmost plot in Figure \ref{fig:reg}).
	The gradient and Hessian of $F(w,v)$ equal
	\[
	\left(
	\begin{matrix}
	(w-1+\epsilon wv)(1+\epsilon v)\\
	(w-1+\epsilon wv)\epsilon w
	\end{matrix}
	\right)~~~\text{and}~~~
	\left(\begin{matrix}
	(1+\epsilon v)^2 & \epsilon(2w+2\epsilon w 
	v-1)\\\epsilon(2w+2\epsilon w v-1) & 
	\epsilon^2 w^2
	\end{matrix}
	\right)
	\]
	respectively. In particular, at $(w,v)=(0,-1/\epsilon)$, the gradient 
	is $\mathbf{0}$ and the Hessian equals
	$\left(\begin{matrix}
	0 & -\epsilon\\ -\epsilon & 0
	\end{matrix}\right)$, which is arbitrarily close to $\mathbf{0}$ if 
	$\epsilon$ is small enough. However, the objective value at that point 
	equals
	\[
	F\left(0,-\frac{1}{\epsilon}\right)~=~\frac{1}{2}>0= F_{lin}(1).
	\]
\end{example}

\begin{remark}
	In the example above, $F$ does not have gradients and Hessians with a 
	uniformly bounded Lipschitz constant (over all of Euclidean space). 
	However, for any $\epsilon>0$, the Lipschitz constants  are bounded by a 
	numerical constant over $(w,v)\in [-2/\epsilon,2/\epsilon]^2$ (which 
	includes the stationary point studied in the construction). This indicates 
	that the problem indeed lies with the norm of $(w,v)$ being unbounded, and 
	not with the Lipschitz constants of the derivatives of $F$.
\end{remark}

One standard approach to ensure that the iterates remain bounded is to add 
regularization, namely optimize 
\[
\min_{\bw,V,\theta}~F(\bw,V,\theta)+R(\bw,V,\theta)~,
\]
where $R$ is a regularization term penalizing large norms of $\bw,V,\theta$. 
Unfortunately, 
not only does this alter the objective, it might also introduce new spurious 
local minima that did not exist in $F(\bw,V,\theta)$. This is graphically 
illustrated in Figure \ref{fig:reg}, which plots $F(w,v)$ from Example 
\ref{example:badstat} (when $\epsilon=1$), with and without regularization of 
the form $R(w,v) = \frac{\lambda}{2}(w^2+v^2)$ where $\lambda=1/2$. Whereas the 
stationary points of 
$F(w,v)$ are either global minima (along two valleys, corresponding to 
$\{(w,v): w(1+\epsilon v)=1\}$) or a 
saddle point (at $(w,v)=(1,-1/\epsilon)$), the regularization created 
a new spurious local minimum around $(w,v)\approx (-1,-1.6)$. Intuitively, this 
is because the regularization makes the objective value increase well before 
the valley of global minima of $F$. Other regularization choices can also lead 
to the same phenomenon. A similar issue can also occur if we impose a hard 
constraint, namely optimize
\[
\min_{\bw,V,\theta:(\bw,V,\theta)\in \mathcal{M}} F(\bw,V,\theta)
\]
for some constrained domain $\mathcal{M}$. Again, as Figure \ref{fig:reg} 
illustrates, this optimization problem can have spurious local 
minima inside its constrained domain, using the same $F$ as before.

Of course, one way to fix this issue is by making the regularization parameter 
$\lambda$ sufficiently small (or the domain $\mathcal{M}$ sufficiently large), 
so that the regularization only comes into 
effect when $\norm{(w,v)}$ is sufficiently large. However, the correct choice 
of $\lambda$ and $\mathcal{M}$ depends on $\epsilon$, and here we run into a 
problem: If $f_{\theta}$ is not simply some fixed $\epsilon$ (as in the example 
above), but changes over time, then we have no a-priori guarantee on how 
$\lambda$ or $\mathcal{M}$ should 
be chosen. Thus, it is not clear that any fixed choice of regularization 
would work, and lead a gradient-based method to a good local minimum. 

\begin{figure}
	\includegraphics[trim=3cm 0cm 3cm 0cm, clip=true,scale=0.6]{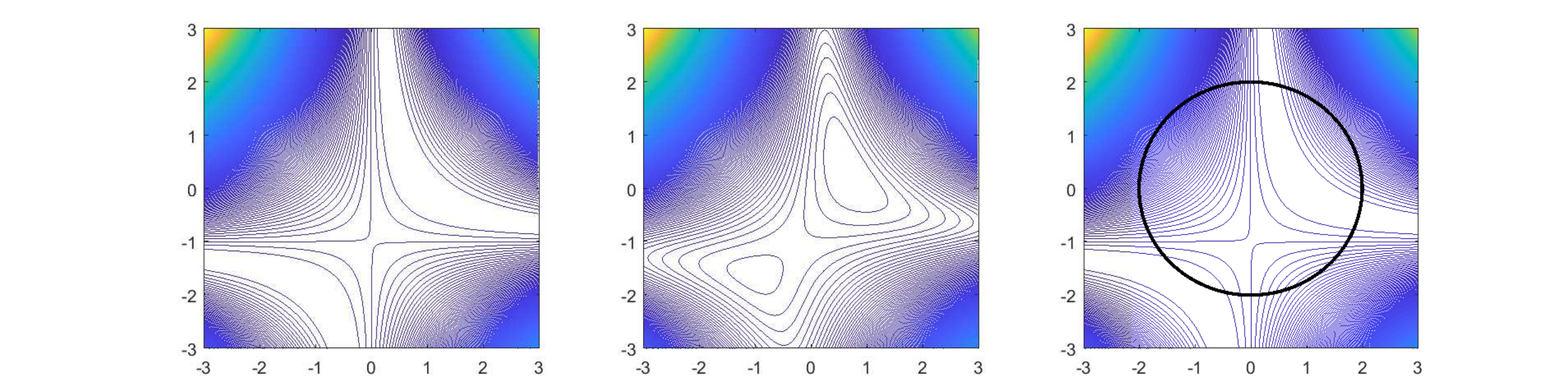}
	\caption{From left to right: Contour plots of (a) $F(w,v)=(w(1+v)-1)^2$, 
	(b) 
	$F(w,v)+\frac{1}{4}(w^2+v^2)$, and (c) $F(w,v)$ superimposed with the 
	constraint $\norm{(w,v)}\leq 2$ (inside the circle). The $x$-axis 
	corresponds 
	to $w$, and the $y$-axis corresponds to $v$. Both (b) and (c) exhibit a 
	spurious local minima in the bottom left quadrant of the domain. 
	Best viewed in color.}\label{fig:reg}
\end{figure}

\section{Success of SGD Assuming a Skip Connection to the 
Output}\label{sec:skipoutput}

Having discussed the challenges of getting an algorithmic result in the 
previous section, we now show how such a result is possible, assuming the 
architecture of our network is changed a bit. 

Concretely, instead of the network architecture $\bx\mapsto 
\bw^\top(\bx+Vf_{\theta}(\bx))$, we consider the architecture 
\[
\bx~\mapsto ~ \bw^\top \bx + \bv^\top f_{\theta}(\bx),
\]
parameterized by vectors $\bw,\bv$ and $\theta$, so our new objective can be 
written as 
\[
F(\bw,\bv,\theta) ~=~ \E_{\bx,y}\left[\ell\left(\bw^\top \bx + \bv^\top 
f_{\theta}(\bx);y\right)\right]~.
\]
This architecture corresponds to having a skip connection directly to the 
network's output, rather than to a final linear output layer. It is similar in 
spirit to the skip-connection studied in \citet{liang2018understanding}, except 
that they had a two-layer nonlinear network instead of our linear  
$\bw^\top\bx$ component.

In what follows, we consider a standard stochastic gradient descent (SGD) 
algorithm to train 
our network: Fixing a step size $\eta$ and some convex parameter domain 
$\mathcal{M}$, we
\begin{enumerate}
	\item Initialize $(\bw_1,\bv_1,\theta_1)$ at some point in $\mathcal{M}$
	\item For $t=1,2,\ldots,T$, we randomly sample a data point $(\bx_t,y_t)$ 
	from the underlying data distribution, and perform
	\[
	(\bw_{t+1},\bv_{t+1},\theta_{t+1}) ~=~ 
	\Pi_{\mathcal{M}}\left((\bw_t,\bv_t,\theta_t)-\eta 
	\nabla 
	h_t(\bw_t,\bv_t,\theta_t)\right),
	\]
	where
	\[
	h_t(\bw,\bv,\theta) ~:=~ \ell(\bw^\top \bx_{t} + \bv^\top 
	f_{\theta}(\bx_t);y_t)
	\]
	and $\Pi_{\mathcal{M}}$ denote an Euclidean projection on the set 
	$\mathcal{M}$. 
\end{enumerate}
Note that $h_t(\bw,\bv,\theta)$ is always differentiable with respect to 
$\bw,\bv$, and in the above, we assume for simplicity that it is also 
differentiable with respect to $\theta$ (if not, one can simply define $\nabla 
h_t(\bw,\bv,\theta)$ above to be $\left(\frac{\partial}{\partial 
\bw}h_t(\bw,\bv,\theta),\frac{\partial}{\partial 
\bv}h_t(\bw,\bv,\theta),\br_{t,\bw,\bv,\theta}\right)$ for some arbitrary 
vector 
$\br_{t,\bw,\bv,\theta}$, and the result below can still be easily verified to 
hold). 

As before, we use the notation 
\[
F_{lin}(\bw) = \E_{\bx,y}\left[\ell\left(\bw^\top\bx;y\right)\right]
\]
to denote the expected loss of a linear predictor parameterized by $\bw$.
The following theorem establishes that under mild conditions, running 
stochastic gradient descent with sufficiently many iterations results in a 
network competitive with any fixed linear predictor:

\begin{theorem}\label{thm:online}
	Suppose the domain $\mathcal{M}$ satisfies the following for some positive 
	constants $b,r,l$:
	\begin{itemize}
		\item $\mathcal{M}=\left\{(\bw,\bv,\theta)~:~(\bw,\bv)\in 
		\mathcal{M}_1,\theta\in \mathcal{M}_2\right\}$ for some closed convex 
		sets $\mathcal{M}_1,\mathcal{M}_2$ in Euclidean spaces (namely, 
		$\mathcal{M}$ is a Cartesian product of $\mathcal{M}_1,\mathcal{M}_2$). 
		\item For any $(\bx,y)$ in the support of the data distribution, and 
		any $\theta\in \mathcal{M}_2$, 
		$\ell(\bw^\top\bx+\bv^\top f_{\theta}(\bx);y)$ is $l$-Lipschitz 
		in $(\bw,\bv)$ over $\mathcal{M}_1$, and bounded in 
		absolute value by $r$.
		\item For any $(\bw,\bv)\in \mathcal{M}_1$, 
		$\sqrt{\norm{\bw}^2+\norm{\bv}^2}\leq b$. 
	\end{itemize}
	Suppose we perform $T$ iterations of stochastic gradient descent as 
	described above, with any step size
	$\eta=\Theta(b/(l\sqrt{T}))$. Then with probability at least $1-\delta$, 
	one 
	of the iterates $\{(\bw_t,\bv_t,\theta_t)\}_{t=1}^{T}$ satisfies
	\[
	F(\bw_t,\bv_t,\theta_t)~\leq~ 
	\min_{\bu:(\bu,\mathbf{0})\in\mathcal{M}_1}F_{lin}(\bu)+
	\Ocal\left(\frac{bl+r\sqrt{\log(1/\delta)}}{\sqrt{T}}\right)~.
	\]
\end{theorem}
The proof relies on a technically straightforward -- but perhaps unexpected -- 
reduction to adversarial online learning, and appears in \secref{sec:proofs}. 
The result can also be easily generalized to the case where the network's 
output is vector valued (see Appendix \ref{app:vec} for a brief discussion).

\section{Proofs}\label{sec:proofs}

\subsection{Proof of \thmref{thm:pl}}

We will utilize the following key lemma, which implies that the inner product 
of the partial derivative at $(\bw,V,\theta)$ with some carefully-chosen vector 
is lower bounded by the suboptimality of $(\bw,V,\theta)$ compared to a linear 
predictor:

\begin{lemma}\label{lem:inproduct}
	Fix some $\bw,V$ (where $\bw\neq \mathbf{0}$) and a vector $\bw^*$ of the 
	same size as $\bw$. Define the matrix
	\[
	G = 
	\left(\bw-\bw^*~;~\frac{1}{\norm{\bw}^2}\bw(\bw^*)^\top 
	V\right)~.
	\]	
	Then
	\[
	\left\langle \text{vec}(G),\nabla 
	F_{\theta}(\bw,V)\right\rangle ~\geq~ 
	F(\bw,V,\theta)-F_{lin}(\bw^*)~.
	\]
\end{lemma}

\begin{proof}
	To simplify notation, let $d_{\ell}=\frac{\partial}{\partial p} \ell(p;y) 
	\vert_{p = \bw^\top (\bx+V 
		f_\theta (\bx))}$. It is easily 
	verified that
	\[
	\frac{\partial}{\partial \bw} F(\bw,V,\theta) = 
	\E_{\bx,y}\left[d_{\ell}(\bx+Vf_{\theta}(\bx))\right].
	\]
	Therefore, we have
	\begin{align}
	&\left\langle 
	\bw-\bw^*~,~\frac{\partial}{\partial \bw} 
	F(\bw,V,\theta) \right\rangle
	~=~ \E_{\bx,y}\left[d_{\ell}(\bw-\bw^*)^\top
	(\bx+Vf_{\theta}(\bx))\right]~.
	\label{eq:traces1}
	\end{align}
	Proceeding in a similar fashion, it is easily verified that
	\[
	\frac{\partial}{\partial V} F(\bw,V,\theta) = 
	\E_{\bx,y}\left[d_{\ell} \bw f_{\theta}(\bx)^\top\right],
	\]
	where we write the partial derivative in matrix form. As a result,
	\begin{align}
	&\left\langle 
	\text{vec}\left(\frac{1}{\norm{\bw}^2}\bw(\bw^*)^\top 
	V\right)~,~\text{vec}\left(\frac{\partial}{\partial
		V} 
	F(\bw,V,\theta)\right) \right\rangle\notag\\
	&~~~~=
	\text{trace}\left(\left(\frac{1}{\norm{\bw}^2}\bw(\bw^*)^\top 
	V\right)^\top 
	\frac{\partial}{\partial
		V}F(\bw,V,\theta)\right)\notag\\
	&~~~~=
	\E_{\bx,y}\left[d_{\ell}~\text{trace}\left(
	\left(\frac{1}{\norm{\bw}^2}V^\top\bw^*\bw^\top\right) \bw 
	f_{\theta}(\bx)^\top\right)\right]\notag\\
	&~~~~=
	\E_{\bx,y}\left[d_{\ell}\text{trace}\left( V^\top \bw^* 
	f_{\theta}(\bx)^\top\right) 
	\right]\notag\\
	&~~~~\stackrel{(*)}{=}\E_{\bx,y}\left[d_{\ell}~ f_{\theta}(\bx)^\top 
	V^\top\bw^*\right]~=~ \E_{\bx,y}\left[d_{\ell}~ (\bw^*)^\top 
	Vf_{\theta}(\bx)\right]~,
	\label{eq:traces2}
	\end{align}
	where in $(*)$ we used the fact that 
	$\text{trace}(ABC)=\text{trace}(CAB)$ for matrices $A,B,C$ that agree in 
	their dimensions. 
	
	Using \eqref{eq:traces1}, \eqref{eq:traces2} and the definition of $G$, it 
	follows that
	\begin{align}
	&\left\langle \text{vec}(G),\text{vec}\left(\nabla
	F_{\theta}(\bw,V)\right)\right\rangle\notag\\
	&=~ 
	\left\langle 
	\bw-\bw^*~,~\frac{\partial}{\partial \bw} 
	F(W,V,\theta) \right\rangle\notag\\
	&~~~~~+
	\left\langle 
	\text{vec}\left(\frac{1}{\norm{\bw}^2}\bw(\bw^*)^\top 
	V\right)~,~\text{vec}\left(\frac{\partial}{\partial
		V} 
	F(W,V,\theta)\right) \right\rangle\notag\\
	&=~
	\E_{\bx,y}\left[d_{\ell}~ (\bw^*)^\top 
	Vf_{\theta}(\bx)\right]
	+\E_{\bx,y}\left[d_{\ell}(\bw-\bw^*)^\top
	(\bx+Vf_{\theta}(\bx))\right]\notag\\
	&=~
	\E_{\bx,y}\left[d_{\ell}\left(\bw^\top(\bx+Vf_{\theta}(\bx))-(\bw^*)^\top\bx\right)\right]~.
	\label{eq:tracesend}
	\end{align}
	Recalling that $d_{\ell}$ is the derivative of $\ell$ with 
	respect to its first argument when it equals $\bw(\bx+Vf_{\theta}(\bx))$, 
	and noting that by convexity of $\ell$, $\frac{\partial}{\partial 
	p}\ell(p;y)(p-\tilde{p})\geq 
	\ell(p;y)-\ell(\tilde{p};y)$ for all $p,\tilde{p}$, it follows that 
	\eqref{eq:tracesend}  is 
	lower bounded by
	\[
	\E_{\bx,y}\left[\ell(\bw^\top(\bx+Vf_{\theta}(\bx);y))-\ell((\bw^*)^\top\bx;y))\right]~=~
	F(\bw,V,\theta)-F_{lin}(\bw^*)~.
	\]
\end{proof}

With this lemma in hand, we turn to prove the theorem. By 
\lemref{lem:inproduct} and Cauchy-Schwartz, we have
\[
\norm{G}_{Fr}\cdot \norm{\nabla F_{\theta}(\bw,V)}
~\geq~\norm{G}\cdot \norm{\nabla F_{\theta}(\bw,V)}~\geq~ 
F(\bw,V,\theta)-F_{lin}(\bw^*)~.
\]
(where $\norm{\cdot}_{Fr}$ denotes the Frobenius norm). 
Dividing both sides by $\norm{G}_{Fr}$, and using the definition of $G$, we 
get 
that
\begin{equation}\label{eq:traces25}
\norm{\nabla F_{\theta}(\bw,V)} ~\geq~ 
\frac{F(\bw,V,\theta)-F_{lin}(\bw^*)}{\sqrt{\norm{\bw-\bw^*}^2+
		\norm{\frac{1}{\norm{\bw}^2}\bw(\bw^*)^\top V}_{Fr}^2}}~.
\end{equation}	
We now simplify this by upper bounding the denominator (note that this leaves 
the inequality valid regardless of the sign of 
$F(\bw,V,\theta)-F_{lin}(\bw^*)$, since if $F(\bw,V,\theta)-F_{lin}(\bw^*)\geq 
0$, this 
would only 
decrease its right hand side, and if $F(\bw,V,\theta)-F_{lin}(\bw^*)< 0$, then 
the bound remains trivially true since $\norm{\nabla
F_{\theta}(\bw,V)}\geq 0> 
\frac{F(\bw,V,\theta)-F_{lin}(\bw^*)}{a}$ for any $a>0$). Specifically, using 
the facts\footnote{The first assertion is standard. The second 
	follows from 
	$\norm{AB}_{Fr}^2=\sum_i 
	\norm{A_iB}^2\leq \sum_i \left(\norm{A_i}\cdot\norm{B}\right)^2 = 
	\norm{B}^2\sum_i\norm{A_{i}}^2=\norm{A}_{Fr}^2\norm{B}^2$, where $A_i$ 
	is the $i$-th row of $A$.}
that $\norm{AB}\leq \norm{A}\cdot\norm{B}$, and that 
$\norm{AB}_{Fr}\leq 
\norm{A}_{Fr}\cdot\norm{B}$, we can upper bound the denominator by
\[
\sqrt{\norm{\bw-\bw^*}^2+\frac{1}{\norm{\bw}^4}\cdot\left(\norm{\bw}\cdot
	\norm{\bw^*}\cdot\norm{V}\right)^2}.
\]
Simplifying the above, using the fact that $\norm{\bw-\bw^*}^2\leq 
2\norm{\bw}^2+2\norm{\bw}^2$ (as for any vectors $\bx,\bz$, 
$\norm{\bx-\bz}^2 \leq \norm{\bx}^2+\norm{\bz}^2+2|\bx^\top\bz| \leq 
\norm{\bx}^2+\norm{\bz}^2+(\norm{\bx}^2+\norm{\bz}^2) =
2(\bx^2+\bz^2)$), and plugging into \eqref{eq:traces25}, the result follows.

\subsection{Proof of \thmref{thm:atzero}}

We will need the following auxiliary lemma:
\begin{lemma}\label{lem:mineig}
	Let $M$ be a symmetric real-valued square matrix of the form
	\[
	M~=~\left(\begin{matrix} b & \bu^\top \\ \bu & 
	\mathbf{0}\end{matrix}\right)~,
	\]	
	where $b$ is some scalar, $\bu$ is a vector, and all entries of $M$ other 
	than the first row and column are $0$. Then the minimal 
	eigenvalue $\lambda_{\min}$ of $M$ is non-positive, and satisfies
	\[
	\norm{\bu}^2~=~ |b\lambda_{\min}|+\lambda_{\min}^2~.
	\]
\end{lemma}
\begin{proof}
	By definition, we have $\lambda_{\min}=\min_{\bz:\norm{\bz}=1}\bz^\top 
	M\bz$. Rewriting $\bz$ as $(x;\by)$ (where $x$ is the first coordinate of 
	$\bz$ and the vector $\by$ represents the other coordinates) and plugging 
	in, this is equivalent to 	
	$\min_{x,\by:x^2+\norm{\by}^2=1} bx^2+2x\bu^\top \by$. Clearly, for any 
	fixed $x$, this is minimized when we take 
	$\by=-a\bu/\norm{\bu}$ (for some $a\geq 0$ satisfying the constraints), so 
	we 
	equivalently have $\lambda_{\min} = \min_{x,a:x^2+a^2=1} 
	bx^2-2\norm{\bu}xa$. This is the same as the minimal eigenvalue of the 
	$2\times 2$ matrix $\left(\begin{matrix} b & 
	\norm{\bu}\\\norm{\bu}&0\end{matrix}\right)$. By standard formulas for 
	$2\times 2$ matrices, it follows that
	$\lambda_{\min}=\frac{1}{2}\left(b-\sqrt{b^2+4\norm{\bu}^2}\right)$. 
	Solving for $\norm{\bu}^2$ and noting that $\lambda_{\min}<0$, the result 
	follows.
\end{proof}

We now turn to prove the theorem. We will use the shorthand $\ell'$ and 
$\ell''$ to denote a derivative and second derivative (respectively) of 
$\ell$ with respect to its first argument. Based on the calculations from 
earlier, 
we have
\begin{equation}\label{eq:derivw}
\frac{\partial}{\partial \bw} F(\bw,V,\theta) = 
\E_{\bx,y}\left[\ell'(\bw^\top(\bx+V 
f_{\theta}(\bx));y)\cdot(\bx+Vf_{\theta}(\bx))\right].
\end{equation}
\[
\frac{\partial}{\partial V} F(\bw,V,\theta) = 
\E_{\bx,y}\left[\ell'(\bw^\top(\bx+V f_{\theta}(\bx));y)\cdot\bw 
f_{\theta}(\bx)^\top\right].
\]
Therefore, for any indices $i,j$, we have
\[
\frac{\partial^2}{\partial V_{i,j}\partial V}F(\bw,V,\theta) = 
\E_{\bx,y}\left[\ell''(\bw^\top(\bx+V f_{\theta}(\bx));y)
\cdot\left(w_i(f_{\theta}(\bx))_j\right)\cdot\bw 
f_{\theta}(\bx)^\top\right]
\]
this is $0$ at $\bw=\mathbf{0}$, hence
\begin{equation}\label{eq:hesspart1}
\frac{\partial^2}{\partial V^2}F(\mathbf{0},V,\theta)~=~\mathbf{0}~.
\end{equation}
Also, 
\begin{align*}
\frac{\partial^2}{\partial w_{i}\partial V}F(\bw,V,\theta)~&=~ 
\E_{\bx,y}\left[\ell''(\bw^\top(\bx+V f_{\theta}(\bx));y)
\cdot (\bx+V f_{\theta}(\bx))_i\cdot\bw 
f_{\theta}(\bx)^\top\right]\\
&~~~+\E_{\bx,y}\left[\ell'(\bw^\top(\bx+V f_{\theta}(\bx));y)
\cdot \be_i
f_{\theta}(\bx)^\top\right]~,
\end{align*}
where $\be_i$ is the $i$-th standard basis vector. 
At $\bw=\mathbf{0}$, this becomes
\begin{equation}\label{eq:hesspart2}
\frac{\partial^2}{\partial w_{i}\partial V}F(\mathbf{0},V,\theta)
~=~
\E_{\bx,y}\left[\ell'(0;y)
\cdot \be_i
f_{\theta}(\bx)^\top\right]
~=~
\be_i\br^\top~,
\end{equation}
where we define 
\begin{equation}\label{eq:rdef}
\br=\E_{\bx,y}[\ell'(0;y)f_{\theta}(\bx)]~.
\end{equation}
Combining 
\eqref{eq:hesspart1} and \eqref{eq:hesspart2}, and recalling that 
$\nabla^2 F_{\theta}(\bw,V)$ is the matrix of second partial derivatives 
of 
$F(\bw,V,\theta)$ with respect to 
$(\bw,V)$ (viewed as one long vector), we get that 
\[
\nabla^2 F_{\theta}(\mathbf{0},V)
~=~\left(
\begin{matrix} 
& & & & \br^\top & \mathbf{0} & \mathbf{0} & \\
& \frac{\partial^2}{\partial\bw^2}F_{\theta}(\mathbf{0},V) 
&  & & \mathbf{0} & \br^\top & \mathbf{0} & \cdots\\
& & & & \mathbf{0} & \mathbf{0} & \br^\top & \\
& & & & & \vdots & & \\
\br & \mathbf{0} & \mathbf{0} & & & & & \\
\mathbf{0} & \br & \mathbf{0} & \cdots & & \mathbf{0} & & \\
\mathbf{0} & \mathbf{0} & \br &  & &  & & \\
& \cdots &&&&&&&
\end{matrix}
\right).
\]
Let $b$ denote the 1st element along the diagonal of 
$\frac{\partial^2}{\partial\bw^2}F(\mathbf{0},V,\theta) $, and define the matrix
\[
M ~=~ \left(\begin{matrix} b & \br^\top \\ \br & 
\mathbf{0}\end{matrix}\right).
\]
This is a submatrix of $\nabla^2 F_{\theta}(\mathbf{0},V)$, which by 
\lemref{lem:mineig}, has a minimal 
eigenvalue $\lambda_{\min}(M)\leq 0$, and 
\begin{equation}\label{eq:eighessian0}
\norm{\br}^2 ~=~ |b\lambda_{\min}(M)|+\lambda_{\min}(M)^2~.
\end{equation}
Since $M$ is a submatrix of $\nabla^2 F_{\theta}(\mathbf{0},V)$, 
we must also have 
\[
	\lambda_{\min}\left(\nabla^2 F_{\theta}(\mathbf{0},V)\right)\leq 
	\lambda_{\min}(M)\leq 0
\]
by the interlacing theorem (proving the first part of the theorem). Finally, 
since $b$ is an element in the diagonal of 
$\frac{\partial^2}{\partial\bw^2}F(\mathbf{0},V,\theta) $, we also have
$b\leq \norm{\frac{\partial^2}{\partial \bw^2} 
	F_{\theta}(\mathbf{0},V)}$. Plugging the above into \eqref{eq:eighessian0}, 
	we get
\begin{equation}\label{eq:eighessian}
\norm{\br}^2~\leq~ 
\left|\lambda_{\min}\left(\nabla^2 
F_{\theta}(\mathbf{0},V)\right)\right|\cdot 
	\left\|\frac{\partial^2}{\partial\bw^2}F_{\theta}(\mathbf{0},V)
	\right\|+\lambda_{\min}
\left(\nabla^2 F_{\theta}(\mathbf{0},V)\right)^2~.
\end{equation}

Leaving this equation aside for a moment, we observe that by \eqref{eq:derivw} 
and the definition of $\br$ from \eqref{eq:rdef}, 
\begin{align*}
\frac{\partial}{\partial \bw} F(\mathbf{0},V,\theta) ~&=~ 
\E_{\bx,y}\left[\ell'(0;y)\cdot(\bx+Vf_{\theta}(\bx))\right]\\
&=~ 
\E_{\bx,y}\left[\ell'(0;y)\bx\right]+\E_{\bx,y}\left[\ell'(0;y)Vf_{\theta}(\bx)\right]
~=~
\nabla F_{lin}(\mathbf{0})+V\br~,
\end{align*}
where we recall that 
$F_{lin}(\bw)=F(\bw,\mathbf{0},\theta)=\E_{\bx,y}\left[\ell(\bw^\top\bx;y)\right]$.
In particular, this 
implies (by the triangle inequality and Cauchy-Shwartz) that 
\[
\left\|\frac{\partial}{\partial\bw} 
F(\mathbf{0},V,\theta)\right\|+\norm{V}\cdot 
\norm{\br}\geq 
\norm{\nabla 
	F_{lin}(\mathbf{0})}~,
\] 
and since $\frac{\partial}{\partial\bw} F(\mathbf{0},V,\theta)$ is a sub-vector 
of $\nabla F_{\theta}(\mathbf{0},V)$, it follows that
\begin{equation}\label{eq:nablalinbound}
\norm{\nabla F_{\theta}(\mathbf{0},V)}+\norm{V}\cdot \norm{\br}~\geq 
\norm{\nabla 
	F_{lin}(\mathbf{0})}~.
\end{equation}
Fixing some $\bw^*$, we have by convexity of $F_{lin}$ that
$
\nabla F_{lin}(\mathbf{0})^\top(\mathbf{0}-\bw^*)\geq 
F_{lin}(\mathbf{0})-F_{lin}(\bw^*)$, and therefore
\[
\norm{\nabla F_{lin}(\mathbf{0})}~\geq~ 
\frac{F_{lin}(\mathbf{0})-F_{lin}(\mathbf{\bw^*})}{\norm{\bw^*}}
~=~
\frac{F(\mathbf{0},V,\theta)-F_{lin}(\mathbf{\bw^*})}{\norm{\bw^*}}
\]
for any $V,\theta$. Combining this with \eqref{eq:nablalinbound} and 
\eqref{eq:eighessian}, and using the shorthand $F_{\theta}$ for 
$F_{\theta}(\mathbf{0},V)$, we get overall that
\[
\norm{\nabla
	F_{\theta}}+\norm{V}\sqrt{\left|\lambda_{\min}\left(\nabla^2 
	F_{\theta}(\mathbf{0},V)\right)\right|\cdot 
		\left\|\frac{\partial^2}{\partial\bw^2}F_{\theta}(\mathbf{0},V)
		\right\|+\lambda_{\min}
	\left(\nabla^2 F_{\theta}(\mathbf{0},V)\right)^2}
~\geq \frac{F(\mathbf{0},V,\theta)-F_{lin}(\mathbf{\bw^*})}{\norm{\bw^*}}
\]
as required.

\subsection{Proof of \thmref{thm:mainstat}}

Letting $\bw^*$ be a minimizer of $F_{lin}$ over $\{\bw:\norm{\bw}\leq r\}$, 
and using the shorthand $F_{\theta}$ for $F_{\theta}(\mathbf{0},V)$, we 
have by \thmref{thm:atzero} and the assumption in our theorem statement that 
$\lambda_{\min}(\nabla^2 F_{\theta})\leq 0$ that
\begin{align*}
F(\mathbf{0},V,\theta) ~&\leq~
F_{lin}(\bw^*)+r \left(\norm{\nabla
	F_{\theta}}+b\sqrt{-\mu_1\lambda_{\min}(\nabla^2 
	F_{\theta})+\lambda_{\min}(\nabla^2 F_{\theta})^2}\right)\\
	&\leq~ F_{lin}(\bw^*)+r \left(\norm{\nabla
		F_{\theta}}+b\sqrt{-\lambda_{\min}(\nabla^2 
		F_{\theta})}\cdot \sqrt{\mu_1-\lambda_{\min}(\nabla^2 
		F_{\theta})}\right)\\
	&\leq~ F_{lin}(\bw^*)+r \left(\norm{\nabla
		F_{\theta}}+b\sqrt{-\lambda_{\min}(\nabla^2 
		F_{\theta})}\cdot \sqrt{\mu_1+\mu_1}\right)\\
&\leq~ F_{lin}(\bw^*)+r \left(\norm{\nabla
	F_{\theta}}+\sqrt{2}b\sqrt{-\mu_1\lambda_{\min}(\nabla^2 
	F_{\theta})}\right)~.
\end{align*}
This inequality refers to $F$ at the point $(\mathbf{0},V,\theta)$. By the 
Lipschitz assumptions in our theorem, it implies that for any $\bw$ such that  
$\norm{\bw}\leq \delta$ for some 
$\delta>0$,
\begin{equation}\label{eq:closezero}
F(\bw,V,\theta)\leq F_{lin}(\bw^*)+\mu_0 \delta+r\left(\norm{\nabla 
	F_{\theta}(\bw,V)}
+\delta\mu_1+\sqrt{2}b\sqrt{-\mu_1\lambda_{\min}(\nabla^2 
	F_{\theta}(\mathbf{0},V))}\right)~.
\end{equation}
On the other hand, by \thmref{thm:pl}, for any 
$\bw$ such that 
$\norm{\bw}>\delta$, we have
\begin{equation}\label{eq:farzero}
F(\bw,V,\theta)~\leq~ F_{lin}(\bw^*)+
\norm{\nabla
	F_{\theta}(\bw,V,\theta)}\cdot\sqrt{2b^2+r^2\left(2+\frac{b^2}{\delta^2}\right)}~.
\end{equation}
Now, let $(\bw,V,\theta)$ be an $\epsilon$-SOPSP
(namely, 
$\norm{\nabla F_{\theta}(\bw,V)}\leq \epsilon$ and 
$\lambda_{\min}(\nabla^2 
F_{\theta}(\bw,V))\geq -\sqrt{\mu_2 \epsilon})$), and note that by the 
Lipschitz 
assumptions and \thmref{thm:atzero}, 
\[
0~\geq~\lambda_{\min}(\nabla^2 
F_{\theta}(\mathbf{0},V))~\geq~ \lambda_{\min}(\nabla^2 
F_{\theta}(\bw,V)-\mu_2\norm{\bw}~\geq -\sqrt{\mu_2\epsilon}-\mu_2\norm{\bw}~.
\]
Combining this 
with \eqref{eq:closezero} and \eqref{eq:farzero}, we get that for any $\delta$, 
if $(\bw,V,\theta)$ is an $\epsilon$-SOPSP, then
\begin{align*}
F(\bw,V,\theta)~\leq&~ F_{lin}(\bw^*)\\
&+\max\left\{~\mu_0 
\delta+r\left(\epsilon+\delta\mu_1+\sqrt{2}b\sqrt{\mu_1(\sqrt{\mu_2\epsilon}
	+\delta\mu_2)}\right)~,\right.\\
&\;\;\;\;\;\;\;\;\;\;\;\;\;\left.
\epsilon\sqrt{2b^2+r^2\left(2+\frac{b^2}{\delta^2}\right)}~\right\}~.
\end{align*}
In particular, if we pick 
$\delta=\sqrt{\epsilon}$, we get
\begin{align*}
F(\bw,V,\theta)~\leq&~ F_{lin}(\bw^*)\\
&+\max\left\{~\mu_0\sqrt{\epsilon} 
+r\left(\epsilon+\mu_1\sqrt{\epsilon}+\sqrt{2}b\sqrt{\mu_1\sqrt{\epsilon}
	(\sqrt{\mu_2}+\mu_2)}\right)~,\right.\\
&\;\;\;\;\;\;\;\;\;\;\;\;\;\left.
\sqrt{2b^2\epsilon^2+r^2\left(2\epsilon^2+b^2\epsilon\right)}~\right\}~.
\end{align*}
Simplifying the above, the right-hand side can be bounded by
\[
F_{lin}(\bw^*)+(\epsilon+\sqrt{\epsilon}+\sqrt[4]{\epsilon})\cdot
\text{poly}(b,r,\mu_0,\mu_1,\mu_2)~.
\]
Since $\max\{\epsilon,\sqrt[4]{\epsilon}\}\geq \sqrt{\epsilon}$ for any 
$\epsilon\geq 0$, 
the result follows.

\subsection{Proof of \thmref{thm:online}}

Consider any fixed sequence of sampled examples 
$(\bx_1,y_1),(\bx_2,y_2),\ldots$. These induce our algorithm to produce a 
fixed series of iterates 
$(\bw_1,\bv_1,\theta_1),(\bw_2,\bv_2,\theta_2),\ldots$. For 
any $t$, define the function
\begin{equation}\label{eq:gdef}
g_t((\bw,\bv)) ~:=~ \ell(\bw^\top\bx_t+\bv^\top f_{\theta_t}(\bx_t);y_t)~,
\end{equation}
where we view $(\bw,\bv)$ as one long vector. 

Our first key observation is the 
following: Our stochastic gradient descent algorithm 
produces 
iterates $(\bw_t,\bv_t)$ which are identical to those produced by the updates
\begin{equation}\label{eq:ogd}
(\bw_{t+1},\bv_{t+1}) = \Pi_{\mathcal{M}_1}\left((\bw_t,\bv_t)-\eta 
\nabla g_t((\bw_t,\bv_t))\right)~,
\end{equation}
where $(\bw_1,\bv_1)\in \mathcal{M}_1$.This follows from the 
fact that $\nabla g_t((\bw,\bv))= 
\frac{\partial}{\partial(\bw,\bv)} h_t(\bw,\bv,\theta_t)$ at 
$(\bw,\bv)=(\bw_t,\bv_t)$, and that the
projection of $(\bw,\bv,\theta)$ on the product set 
$\mathcal{M}=\mathcal{M}_1\times\mathcal{M}_2$ is equivalent to seperately 
projecting $(\bw,\bv)$ on $\mathcal{M}_1$ and $\theta$ on $\mathcal{M}_2$. 

Our second key observation is that the iterates as defined in \eqref{eq:ogd} 
are identical to the iterates produced by an \emph{online} 
gradient descent algorithm (with step size $\eta$) with 
respect to the sequence of functions $g_1,g_2,\ldots$ 
\citep{zinkevich2003online,shalev2012online,hazan2016introduction}. In 
particular, the following theorem is well-known (see references above): 
\begin{theorem}\label{thm:onlinegd}
	Let $g_1,\ldots,g_T$ be a sequence of convex, differentiable, $l$-Lipschitz 
	functions over a closed convex subset $\Xcal$ of Euclidean space, such that 
	$\Xcal\subseteq\{\bx:\norm{\bx}\leq 
	b\}$. Then if we pick $\bx_1\in \Xcal$ and define 
	$\bx_{t+1}=\Pi_{\Xcal}\left(\bx_{t}-\eta \nabla g_t(\bx_t)\right)$ for 
	$t=1,\ldots,T$, using any $\eta=\Theta(b/(l\sqrt{T}))$, then
	\[
	\forall 
	\bx\in\Xcal~,~~~\frac{1}{T}\sum_{t=1}^{T}g_t(\bx_t)-\frac{1}{T}\sum_{t=1}^{T}g_t(\bx)
	~\leq~\Ocal\left(\frac{bl}{\sqrt{T}}\right)~.
	\]
\end{theorem}
In particular, running the gradient descent updates on any sequence of convex 
Lipschitz functions $g_1,g_2,\ldots$ implies that the average of $g_t(\bx_t)$ 
is not much higher than the minimal value of $\frac{1}{T}\sum_t g_t(\bx)$ over 
$\bx\in\Xcal$. We now argue that this bound is directly applicable to the 
functions $g_t$ 
defined in \eqref{eq:gdef}, over the domain $\mathcal{M}_1$: Indeed,   
since $\ell$ is convex and assumed to be $l$-Lipschitz in $(\bw,\bv)$, each 
function $g_t$ is also convex and $l$-Lipschitz.
Moreover, the elements in 
$\mathcal{M}_1$ (viewed as 
one long vector) are assumed to have an Euclidean norm of at most $b$, and the 
updates as defined in \eqref{eq:ogd} are the same as in \thmref{thm:onlinegd}. 
Therefore, applying  the theorem, we get
\[
\forall (\bw,\bv)\in \mathcal{M}_1~,~~~ 
\frac{1}{T}\sum_{t=1}^{T}g_t((\bw_t,\bv_t))~-~
\frac{1}{T}\sum_{t=1}^{T}g_t((\bw,\bv))~\leq~\Ocal\left(\frac{bl}{\sqrt{T}}\right)~.
\]
We emphasize that this result holds 
\emph{deterministically} regardless of how the examples $(\bx_t,y_t)$ are 
sampled. Slightly rearranging and plugging back 
the definition of $g_t$, we get that
\[
\forall (\bw,\bv)\in \mathcal{M}_1~,~~~ 
\frac{1}{T}\sum_{t=1}^{T}\left(\ell(\bw_t^\top\bx_t+\bv_t^\top 
f_{\theta_t}(\bx_t);y_t)
-\ell(\bw^\top\bx_t+\bv^\top f_{\theta_t}(\bx_t);y_t)\right)
~\leq~ \Ocal\left(\frac{bl}{\sqrt{T}}\right)~.
\]	
In particular, considering any $\bw$ such that $(\bw,\mathbf{0})\in  
\mathcal{M}_1$, we get that
\begin{equation}\label{eq:ogdreg}
	\frac{1}{T}\sum_{t=1}^{T}\left(\ell(\bw_t\bx_t+\bv_t^\top 
	f_{\theta_t}(\bx_t);y_t)
	-\ell(\bw^\top\bx_t;y_t)\right)
	~\leq~ \Ocal\left(\frac{bl}{\sqrt{T}}\right)~.
\end{equation}
Again, this inequality holds deterministically. Now, note that each 
$(\bx_t,y_t)$ is a fresh sample, independent of the
history up to 
round $t$, and conditioned on this history, the expectation of 
$\ell(\bw_t^\top \bx+\bv_t^\top 
f_{\theta_t(\bx_t)};y_t)-\ell(\bw^\top\bx_t;y_t)$ (over 
$(\bx_t,y_t)$) equals $F(\bw_t,\bv_t,\theta_t)-F_{lin}(\bw)$. Therefore, 
\[
\left(F(\bw_t,\bv_t,\theta_t)-F_{lin}(\bw)\right)-\left(\ell(\bw_t^\top 
\bx+\bv_t^\top 
f_{\theta_t(\bx_t)};y_t)-\ell(\bw^\top\bx_t;y_t)\right)
\]
is a martingale difference sequence. Moreover, we assume that the losses $\ell$ 
are bounded by $r$, so by Azuma's inequality, with probability at least 
$1-\delta$, 
the average of the expression above over $t=1,\ldots,T$ is at most 
$\Ocal(r\sqrt{\log(1/\delta)/T})$. Combining this with \eqref{eq:ogdreg}, we 
get that with probability at least 
$1-\delta$,
\[
\frac{1}{T}\sum_{t=1}^{T}\left(F(\bw_t,\bv_t,\theta_t)-F_{lin}(\bw)\right)
~\leq~ 
\Ocal\left(\frac{bl}{\sqrt{T}}+r\sqrt{\frac{\log(1/\delta)}{T}}\right)~.
\]
Rearranging this, we get
\[
\frac{1}{T}\sum_{t=1}^{T}F(\bw_t,\bv_t,\theta_t)~\leq~
F_{lin}(\bw)+\Ocal\left(\frac{bl+r\sqrt{\log(1/\delta)}}{\sqrt{T}}\right)~.
\]
Since the left-hand side is an average over $T$ terms, at least one of those 
terms must be upper bounded by the right-hand side, so there exists some $t$ 
such that
\[
F(\bw_t,\bv_t,\theta_t)~\leq~
F_{lin}(\bw)+\Ocal\left(\frac{bl+r\sqrt{\log(1/\delta)}}{\sqrt{T}}\right)~.
\]
Finally, since this holds for any fixed $\bw$ such that $(\bw,\mathbf{0})\in 
\mathcal{M}_1$, we can choose the one minimizing the right hand side, from 
which the result follows. 

\subsection*{Acknowledgements}

We thank the anonymous NIPS 2018 reviewers for their helpful comments. This  
research is supported in part by European Research Council (ERC) grant  
754705.

\bibliographystyle{plainnat}
\bibliography{mybib}

\begin{thebibliography}{23}
\providecommand{\natexlab}[1]{#1}
\providecommand{\url}[1]{\texttt{#1}}
\expandafter\ifx\csname urlstyle\endcsname\relax
  \providecommand{\doi}[1]{doi: #1}\else
  \providecommand{\doi}{doi: \begingroup \urlstyle{rm}\Url}\fi

\bibitem[Bartlett et~al.(2018)Bartlett, Helmbold, and
  Long]{bartlett2018gradient}
Peter~L Bartlett, David~P Helmbold, and Philip~M Long.
\newblock Gradient descent with identity initialization efficiently learns
  positive definite linear transformations by deep residual networks.
\newblock \emph{arXiv preprint arXiv:1802.06093}, 2018.

\bibitem[Brutzkus et~al.(2017)Brutzkus, Globerson, Malach, and
  Shalev-Shwartz]{brutzkus2017sgd}
Alon Brutzkus, Amir Globerson, Eran Malach, and Shai Shalev-Shwartz.
\newblock Sgd learns over-parameterized networks that provably generalize on
  linearly separable data.
\newblock \emph{arXiv preprint arXiv:1710.10174}, 2017.

\bibitem[Du and Lee(2018)]{du2018power}
Simon~S Du and Jason~D Lee.
\newblock On the power of over-parametrization in neural networks with
  quadratic activation.
\newblock \emph{arXiv preprint arXiv:1803.01206}, 2018.

\bibitem[Du et~al.(2017)Du, Lee, Tian, Poczos, and Singh]{du2017gradient}
Simon~S Du, Jason~D Lee, Yuandong Tian, Barnabas Poczos, and Aarti Singh.
\newblock Gradient descent learns one-hidden-layer cnn: Don't be afraid of
  spurious local minima.
\newblock \emph{arXiv preprint arXiv:1712.00779}, 2017.

\bibitem[Ge and Ma(2017)]{ge2017optimization}
Rong Ge and Tengyu Ma.
\newblock On the optimization landscape of tensor decompositions.
\newblock In \emph{Advances in Neural Information Processing Systems}, pages
  3656--3666, 2017.

\bibitem[Ge et~al.(2017)Ge, Lee, and Ma]{ge2017learning}
Rong Ge, Jason~D Lee, and Tengyu Ma.
\newblock Learning one-hidden-layer neural networks with landscape design.
\newblock \emph{arXiv preprint arXiv:1711.00501}, 2017.

\bibitem[Hardt and Ma(2016)]{hardt2016identity}
Moritz Hardt and Tengyu Ma.
\newblock Identity matters in deep learning.
\newblock \emph{arXiv preprint arXiv:1611.04231}, 2016.

\bibitem[Hazan(2016)]{hazan2016introduction}
Elad Hazan.
\newblock Introduction to online convex optimization.
\newblock \emph{Foundations and Trends{\textregistered} in Optimization},
  2\penalty0 (3-4):\penalty0 157--325, 2016.

\bibitem[He et~al.(2016{\natexlab{a}})He, Zhang, Ren, and Sun]{he2016deep}
Kaiming He, Xiangyu Zhang, Shaoqing Ren, and Jian Sun.
\newblock Deep residual learning for image recognition.
\newblock In \emph{Proceedings of the IEEE conference on computer vision and
  pattern recognition}, pages 770--778, 2016{\natexlab{a}}.

\bibitem[He et~al.(2016{\natexlab{b}})He, Zhang, Ren, and Sun]{he2016identity}
Kaiming He, Xiangyu Zhang, Shaoqing Ren, and Jian Sun.
\newblock Identity mappings in deep residual networks.
\newblock In \emph{European Conference on Computer Vision}, pages 630--645.
  Springer, 2016{\natexlab{b}}.

\bibitem[Jin et~al.(2017)Jin, Ge, Netrapalli, Kakade, and
  Jordan]{jin2017escape}
Chi Jin, Rong Ge, Praneeth Netrapalli, Sham~M Kakade, and Michael~I Jordan.
\newblock How to escape saddle points efficiently.
\newblock \emph{arXiv preprint arXiv:1703.00887}, 2017.

\bibitem[Kim et~al.(2016)Kim, Kwon~Lee, and Mu~Lee]{kim2016accurate}
Jiwon Kim, Jung Kwon~Lee, and Kyoung Mu~Lee.
\newblock Accurate image super-resolution using very deep convolutional
  networks.
\newblock In \emph{Proceedings of the IEEE Conference on Computer Vision and
  Pattern Recognition}, pages 1646--1654, 2016.

\bibitem[Liang et~al.(2018)Liang, Sun, Li, and Srikant]{liang2018understanding}
Shiyu Liang, Ruoyu Sun, Yixuan Li, and R~Srikant.
\newblock Understanding the loss surface of neural networks for binary
  classification.
\newblock \emph{arXiv preprint arXiv:1803.00909}, 2018.

\bibitem[McCormick(1977)]{mccormick1977modification}
Garth~P McCormick.
\newblock A modification of armijo's step-size rule for negative curvature.
\newblock \emph{Mathematical Programming}, 13\penalty0 (1):\penalty0 111--115,
  1977.

\bibitem[Nesterov and Polyak(2006)]{nesterov2006cubic}
Yurii Nesterov and Boris~T Polyak.
\newblock Cubic regularization of newton method and its global performance.
\newblock \emph{Mathematical Programming}, 108\penalty0 (1):\penalty0 177--205,
  2006.

\bibitem[Safran and Shamir(2017)]{safran2017spurious}
Itay Safran and Ohad Shamir.
\newblock Spurious local minima are common in two-layer relu neural networks.
\newblock \emph{arXiv preprint arXiv:1712.08968}, 2017.

\bibitem[Shalev-Shwartz(2012)]{shalev2012online}
Shai Shalev-Shwartz.
\newblock Online learning and online convex optimization.
\newblock \emph{Foundations and Trends{\textregistered} in Machine Learning},
  4\penalty0 (2):\penalty0 107--194, 2012.

\bibitem[Soltanolkotabi et~al.(2017)Soltanolkotabi, Javanmard, and
  Lee]{soltanolkotabi2017theoretical}
Mahdi Soltanolkotabi, Adel Javanmard, and Jason~D Lee.
\newblock Theoretical insights into the optimization landscape of
  over-parameterized shallow neural networks.
\newblock \emph{arXiv preprint arXiv:1707.04926}, 2017.

\bibitem[Soudry and Hoffer(2017)]{soudry2017exponentially}
Daniel Soudry and Elad Hoffer.
\newblock Exponentially vanishing sub-optimal local minima in multilayer neural
  networks.
\newblock \emph{arXiv preprint arXiv:1702.05777}, 2017.

\bibitem[Xie et~al.(2017)Xie, Girshick, Doll{\'a}r, Tu, and
  He]{xie2017aggregated}
Saining Xie, Ross Girshick, Piotr Doll{\'a}r, Zhuowen Tu, and Kaiming He.
\newblock Aggregated residual transformations for deep neural networks.
\newblock In \emph{Computer Vision and Pattern Recognition (CVPR), 2017 IEEE
  Conference on}, pages 5987--5995. IEEE, 2017.

\bibitem[Xiong et~al.(2017)Xiong, Droppo, Huang, Seide, Seltzer, Stolcke, Yu,
  and Zweig]{xiong2017microsoft}
Wayne Xiong, Jasha Droppo, Xuedong Huang, Frank Seide, Mike Seltzer, Andreas
  Stolcke, Dong Yu, and Geoffrey Zweig.
\newblock The microsoft 2016 conversational speech recognition system.
\newblock In \emph{Acoustics, Speech and Signal Processing (ICASSP), 2017 IEEE
  International Conference on}, pages 5255--5259. IEEE, 2017.

\bibitem[Yun et~al.(2018)Yun, Sra, and Jadbabaie]{yun2018critical}
Chulhee Yun, Suvrit Sra, and Ali Jadbabaie.
\newblock A critical view of global optimality in deep learning.
\newblock \emph{arXiv preprint arXiv:1802.03487}, 2018.

\bibitem[Zinkevich(2003)]{zinkevich2003online}
Martin Zinkevich.
\newblock Online convex programming and generalized infinitesimal gradient
  ascent.
\newblock In \emph{Proceedings of the 20th International Conference on Machine
  Learning (ICML-03)}, pages 928--936, 2003.

\end{thebibliography}

\appendix

\section{Vector-Valued Networks}\label{app:vec}

In our paper, we focused on the case of neural networks with scalar-valued 
outputs, and losses over scalar values.
However, for tasks such as multi-class classification, it is common to use 
networks with vector-valued outputs $\mathbf{p}$ (representing a score for each 
possible class), and losses $\ell(\mathbf{p},\by)$ taking as input vector 
values (for example, the cross-entropy loss). Thus, it is natural to ask 
whether our 
results can be extended to such a setting.

In this setting, our objective function from \eqref{eq:obj} takes the form
\[
F(W,V,\theta) ~=~ 
\E_{\bx,\by}\left[\ell\left(W(\bx+Vf_{\theta}(\bx))\right);\by\right],
\]
where $W$ is a matrix, and as before, we define
\[
F_{\theta}(W,V) ~:=~ F(W,V,\theta)
\]
where $\theta$ is considered fixed. Using a similar proof technique, it is 
possible to prove a generalization of \thmref{thm:pl} for this case:

\begin{theorem}\label{thm:plvec}
	Suppose that $F$ is defined as above, where $\ell$ is differentiable and 
	convex in its first 
	argument. Then at any point $(W,V,\theta)$, for which $W$ has full row 
	rank and minimal singular value $s_{\min}(W)>0$, it holds that
	\[
	\norm{\nabla F_{\theta}(W,V)} ~\geq~ 
	\frac{F(W,V,\theta)-F_{lin}(W^*)}{\sqrt{2\norm{W}_{Fr}^2+\norm{W^*}_{Fr}^2\left(2+\frac{\norm{V}^2}
			{s_{\min}(W)^2}\right)}}~.
	\]
\end{theorem}

\begin{proof}
	The proof proceeds in the same manner as in \thmref{thm:pl} (where $W$ was 
	a vector). We will need the following key lemma, whose proof is provided 
	below:
\begin{lemma}\label{lem:inproductvec}
	Fix some $W,V$ and a matrix $W^*$ of the same size as $W$. Define 
	the matrix
	\[
	G = 
	\left(W-W^*~;~W^\top(WW^\top)^{-1}W^*V\right)~.
	\]	
	Then
	\[
	\left\langle \text{vec}(G),\nabla 
	F_{\theta}(W,V)\right\rangle ~\geq~ 
	F(W,V,\theta)-F_{lin}(W^*)~.
	\]
\end{lemma}

By 
\lemref{lem:inproductvec} and Cauchy-Schwartz, we have
\[
\norm{G}_{Fr}\cdot \norm{\nabla F_{\theta}(W,V)}~\geq~ 
F(W,V,\theta)-F_{lin}(W^*).
\]
Dividing both sides by $\norm{G}_{Fr}$, and using the definition of $G$, we 
get 
that
\begin{equation}\label{eq:traces25vec}
\norm{\nabla F_{\theta}(W,V)} ~\geq~ 
\frac{F(W,V,\theta)-F_{lin}(W^*)}{\sqrt{\norm{W-W^*}_{Fr}^2+
		\norm{W^\top(WW^\top)^{-1}W^*V}_{Fr}^2}}
\end{equation}	
As in the proof of \thmref{thm:pl}, we can simplify this bound by upper 
bounding the denominator. 
Using the facts that $\norm{AB}\leq \norm{A}\cdot\norm{B}$, that 
$\norm{AB}_{Fr}\leq 
\norm{A}\cdot\norm{B}_{Fr}$, and that $\norm{AB}_{Fr}\leq 
\norm{A}_{Fr}\norm{B}$, this denominator is at most
\begin{equation}\label{eq:traces3vec}
\sqrt{\norm{W-W^*}_{Fr}^2+\left(\norm{W^\top (WW^\top)^{-1}}\cdot
	\norm{W^*}_{Fr}\norm{V}\right)^2}.
\end{equation}
It is easily verified that if $USV^\top$ is the SVD decomposition 
of $W$, then $W^\top(WW^\top)^{-1}$ equals $VS(SS)^{-1}U^\top$. Since $V,U$ are 
orthogonal, it follows that
$\norm{W(WW^\top)^{-1}}=\norm{S(SS)^{-1}}=s_{\min}(W)^{-1}$, where 
$s_{\min}(W)$ 
is the smallest singular value of $W$. Moreover, we have 
$\norm{W-W^*}_{Fr}^2\leq 
2\norm{W}_{Fr}^2+2\norm{W}_{Fr}$. Therefore, \eqref{eq:traces3vec} is at most
\[
\sqrt{2\norm{W}_{Fr}^2+2\norm{W^*}_{Fr}^2+\left(s_{\min}(W)^{-1}
	\norm{W^*}_{Fr}\norm{V}\right)^2}.
\]
Slightly simplifying and plugging back into \eqref{eq:traces25vec}, the 
result follows.	
\end{proof}

\begin{proof}[Proof of \lemref{lem:inproductvec}]
	To simplify notation, let $\bd_{\ell}$ denote the (vector-valued) gradient 
	of 
	$\ell$ with respect to its 
	first argument at the point 
	$W(\bx+Vf_{\theta}(\bx))$ and $\by$. For any $i,j$, it is easily verified 
	that
	\[
	\frac{\partial}{\partial W_{i,j}} F(W,V,\theta) = 
	\E_{\bx,\by}\left[(\bd_{\ell})_i
	(\bx+Vf_{\theta}(\bx))_j\right]~,
	\] 
	from which it follows that
	\[
	\frac{\partial}{\partial W} F(W,V,\theta) = 
	\E_{\bx,\by}\left[\bd_{\ell}(\bx+Vf_{\theta}(\bx))^\top\right]~,
	\]
	where we write the partial derivative in matrix form. 
	Therefore, we have
	\begin{align}
	&\left\langle \text{vec}(W-W^*),\text{vec}\left(\frac{\partial}{\partial W} 
	F(W,V,\theta)\right) \right\rangle
	= \text{trace}\left((W-W^*)^\top \frac{\partial}{\partial W} 
	F(W,V,\theta)\right)\notag\\
	&~~~~~= \E_{\bx,\by}\left[\text{trace}\left((W-W^*)^\top 
	\bd_{\ell}(\bx+Vf_{\theta}(\bx))^\top\right)\right]\notag\\
	&~~~~~= 
	\E_{\bx,\by}\left[\text{trace}\left((\bx+Vf_{\theta}(\bx))^\top(W-W^*)^\top 
	\bd_{\ell}\right)\right]\notag\\
	&~~~~~= 
	\E_{\bx,\by}\left[\bd_{\ell}^\top(W-W^*)(\bx+Vf_{\theta}(\bx))\right]~,
	\label{eq:traces1vec}
	\end{align}
	where we used the facts that $\text{trace}(ABC)=\text{trace}(CAB)$ and 
	$\text{trace}(A)=\text{trace}(A^\top)$. 
	
	Proceeding in a similar fashion, it is straightforward to verify that
	\[
	\frac{\partial}{\partial V_{i,j}} F(W,V,\theta) = 
	\E_{\bx,\by}\left[(\bd_{\ell})^\top W_{ 
		i}(f_{\theta}(\bx))_j\right]
	\]
	(where $W_{i}$ is the $i$-th column of $W$), and therefore
	\[
	\frac{\partial}{\partial V} F(W,V,\theta) = 
	\E_{\bx,\by}\left[W^\top 
	(\bd_{\ell})f_{\theta}(\bx)^\top\right].
	\]
	As a result,
	\begin{align}
	&\left\langle 
	\text{vec}(W^\top(WW^\top)^{-1}W^*V~,~\text{vec}\left(\frac{\partial}{\partial
		V} 
	F(W,V,\theta)\right) \right\rangle\notag\\
	&~~~~=
	\text{trace}\left(\left(W^\top(WW^\top)^{-1}W^*V\right)^\top 
	\frac{\partial}{\partial
		V}F(W,V,\theta)\right)\notag\\
	&~~~~=
	\E_{\bx,\by}\left[\text{trace}\left((W^*V)^{\top}(WW^\top)^{-1}WW^\top 
	(\bd_{\ell}) f_{\theta}(\bx)^\top\right)
	\right]\notag\\
	&~~~~=
	\E_{\bx,\by}\left[\text{trace}\left(f_{\theta}(\bx)^\top(W^*V)^{\top}
	(\bd_{\ell})\right)
	\right]\notag\\	
	&~~~~=\E_{\bx,\by}\left[ (\bd_{\ell})^\top W^*Vf_{\theta}(\bx)\right]~.
	\label{eq:traces2vec}
	\end{align}
	Summing \eqref{eq:traces1vec} and \eqref{eq:traces2vec}, and recalling the 
	definition of $G$, it follows that
	\begin{align*}
	&\left\langle \text{vec}(G),\nabla 
	F_{\theta}(W,V)\right\rangle~=~\notag\\
	&=~
	\left\langle \text{vec}(W-W^*),\text{vec}\left(\frac{\partial}{\partial W} 
	F(W,V,\theta)\right) \right\rangle
	+
	\left\langle 
	\text{vec}(W^\top(WW^\top)^{-1}W^*V~,~\text{vec}\left(\frac{\partial}{\partial
		V} 
	F(W,V,\theta)\right) \right\rangle\notag\\
	&=~
	\E_{\bx,\by}\left[\bd_{\ell}^\top(W-W^*)(\bx+Vf_{\theta}(\bx))\right]+
	\E_{\bx,\by}\left[ (\bd_{\ell})^\top W^*Vf_{\theta}(\bx)\right]\\
	&=~
	\E_{\bx,\by}\left[(\bd_{\ell})^\top 
	(W(\bx+Vf_{\theta}(\bx))-W^*\bx)\right]~.
	\end{align*}
	Recalling that $ \bd_{\ell} $ is the gradient of $\ell$ with 
	respect to its first argument when it equals $W(\bx+Vf_{\theta}(\bx))$, it 
	follows by convexity of $\ell$ that the expression above is lower bounded by
	\[
	\E_{\bx,\by}\left[\ell(W(\bx+Vf_{\theta}(\bx);\by))-\ell(W^*\bx;\by)\right]~=~
	F(W,V,\theta)-F_{lin}(W^*)~.
	\]
\end{proof}

This theorem already establishes that our objective in the vector-valued case 
has no stationary point $(W,V,\theta)$ (with respect to $(W,V)$) with values 
above 
$F_{lin}(W^*)$, except possibly when 
$W$ is not full row rank, or $s_{\min}(W)=0$. To analyze those cases, one would 
need an analog of \thmref{thm:atzero}, but unfortunately it is currently 
unclear how to prove such a result. We leave this question to future research.

Finally, we note that it is straightforward to derive a vector-valued version 
of \thmref{thm:online} (using stochastic gradient descent over the matrices 
$W,V$ instead of $\bw,\bv$), using a virtually identical proof.

\end{document}